%% file: main.tex
\documentclass[sigconf]{acmart}
\AtBeginDocument{%
  \providecommand\BibTeX{{%
    \normalfont B\kern-0.5em{\scshape i\kern-0.25em b}\kern-0.8em\TeX}}}
    
\usepackage{subfigure, graphicx}

\newtheorem{theorem}{Theorem}

\newtheorem{innercustomthmc}{Corollary}
\newenvironment{customthmc}[1]
  {\renewcommand\theinnercustomthmc{#1}\innercustomthmc}
  {\endinnercustomthmc}

\copyrightyear{2022}
\acmYear{2022}
\setcopyright{acmcopyright}
\acmConference[WSDM '22] {Proceedings of the Fifteenth ACM International Conference on Web Search and Data Mining}{February 21--25, 2022}{Tempe, AZ, USA.}
\acmBooktitle{Proceedings of the Fifteenth ACM International Conference on Web Search and Data Mining (WSDM '22), February 21--25, 2022, Tempe, AZ, USA}
\acmPrice{15.00}
\acmISBN{978-1-4503-9132-0/22/02}
\acmDOI{10.1145/3488560.3498474}

\begin{document}
\fancyhead{}

\title{Understanding and Improvement of Adversarial Training for Network Embedding from an Optimization Perspective}

\author{Lun Du}
\authornote{Equal contribution}
\authornote{Corresponding author}
\email{lun.du@microsoft.com}
\affiliation{%
  \institution{Microsoft Research Asia}
  \city{Beijing}
  \country{China}
}

\author{Xu Chen}
\authornotemark[1]
\authornote{Work performed during their intership in MSRA}
\email{sylover@pku.edu.cn}
\affiliation{%
  \institution{Peking University}
  \city{Beijing}
  \country{China}
}

\author{Fei Gao}
\authornotemark[3]
\email{feig@mail.bnu.edu.cn}
\affiliation{%
  \institution{Beijing Normal University}
  \city{Beijing}
  \country{China}
}

\author{Qiang Fu}
\email{qifu@microsoft.com}
\affiliation{%
  \institution{Microsoft Research Asia}
  \city{Beijing}
  \country{China}
}

\author{Kunqing Xie}
\authornotemark[2]
\email{kunqing@pku.edu.cn}
\affiliation{%
  \institution{Peking University}
  \city{Beijing}
  \country{China}
}

\author{Shi Han}
\author{Dongmei Zhang}
\email{{shihan, dongmeiz}@microsoft.com}
\affiliation{%
  \institution{Microsoft Research Asia}
  \city{Beijing}
  \country{China}
}



\begin{abstract}
  Network Embedding aims to learn a function mapping the nodes to Euclidean space contribute to multiple learning analysis tasks on networks. However, both the noisy information behind the real-world networks and the overfitting problem negatively impact the quality of embedding vectors. To tackle these problems, researchers utilize Adversarial Perturbations on Parameters (APP) and achieve state-of-the-art performance. Unlike the mainstream methods introducing perturbations on the network structure or the data feature, Adversarial Training for Network Embedding (AdvTNE) adopts APP to directly perturb the model parameters, thus provides a new chance to understand the mechanism behind it. In this paper, we explain APP theoretically from an optimization perspective. Considering the Power-law property of networks and the optimization objective, we analyze the reason for its remarkable results on network embedding. Based on the above analysis and the Sigmoid saturation region problem, we propose a new Sine-base activation to enhance the performance of AdvTNE. We conduct extensive experiments on four real networks to validate the effectiveness of our method in node classification and link prediction. The results demonstrate that our method is competitive with state-of-the-art methods.
\end{abstract}

\ccsdesc[500]{Networks~Topology analysis and generation}
\ccsdesc[300]{Networks~Physical topologies}
\ccsdesc[300]{Information systems~Network data models}
\ccsdesc[300]{Mathematics of computing~Mathematical optimization}

\keywords{network embedding, adversarial training, optimization method, saturation region problem}

\maketitle
\input{src/intro}
\input{src/model}

\input{src/related_work}

\input{src/conclusion}

\bibliographystyle{ACM-Reference-Format}
\bibliography{sample-base}

\input{src/appendix}
\end{document}

%% file: src/intro.tex
\section{Introduction}

Graph data play an essential role in social life \cite{cook2006mining,du2018traffic,song2020inferring}. However, numerous graph data contain only the network topology structure, meaning the lack of auxiliary information like node features and edge features. Designed to learn representations of nodes, edges or graphs under this scenario, network embedding methods \cite{Perozzi2014DeepWalk,Tang2015LINE,grover2016node2vec,Cao2015GraRep,gf,9414919,10.1145/3434747} have gained rapid progress in recent years. Based on customized definitions of node neighbors, different unsupervised loss functions are designed to project the nodes from the input space into a lower-dimensional vector space while preserving the structure information. Nevertheless, noises behind the real-world networks and the overfitting problem present new challenges in network embedding.

Adversarial training, initially designed to defend adversarial attacks in the area of computer vision \cite{goodfellow2014explaining}, has shown effectiveness in improving the robustness of deep learning models. Thus, it is extensively used in a variety of domains such as computer vision \cite{Miyato2015distributional,madry2017towards,shafahi2019adversarial}, speech recognition \cite{sun2018domain,liu2019adversarial,drexler2018combining} and natural language processing \cite{miyato2018virtual,zhu2019freelb,liu2020adversarial}. Previous studies focus on directly introduce adversarial perturbations to input space, forcing the estimated distribution to be smoother and hence increases the adversarial robustness of different methods \cite{miyato2018virtual}. Tempted by the success of adversarial training in a wide range of areas, researchers begin to apply adversarial training to the network embedding area for better performance. However, most works follow the manner that adding perturbations on the input space, including  the network topology structure \cite{dropedge,topology, 142adversarial} or data features \cite{31batch,40graph,89virtual,105adversarial, edgefeature, latentfeature}, leaving \textbf{A}dversarial \textbf{P}erturbations on \textbf{P}arameters (APP) less studied.

Adversarial Training for Network Embedding (AdvTNE) \cite{dai2019adversarial}, to the best of our knowledge, is the first research that fills in this gap and has attracted widespread attention from the machine learning communities. AdvTNE introduces a disruptive change to conventional adversarial training on the network by adding adversarial perturbations to the model parameters and achieves impressive results. Following this study, APP is applied in computer vision in turn as a regularization method for better generalization ability \cite{wu2020adversarial,zheng2020regularizing,foret2020sharpness}. Compared with the conventional perturbation-on-input methods that can be interpreted from the perspective of robust learning or function smoothing \cite{miyato2018virtual}, the start-up APP-based methods lack an in-depth understanding. On the one hand, interpretability remains a desire yet challenge in the machine learning domain, and it is essential since it will enable better improvement against models and promote the development of the machine learning community. On the other hand, adversarial perturbations on model parameters demonstrates an entirely new paradigm of adversarial training and shows great potential in promoting model effectiveness, especially for those solving transductive embedding problems. A reasonable explanation is in urgent need to exhibit why and where APP can work well, and thus we are able to determine whether APP can be generalized to other similar scenarios such as word embedding \cite{Mikolov2013Distributed,levy2014linguistic}, tag embedding \cite{wang2019tag2vec, wang2019tag2gauss}, table embedding \cite{gentile2017entity,zhang2019table2vec,10.1145/3447548.3467228} and broader domains. 

In this paper, we theoretically analyze APP from an optimization perspective to explore the reason for its remarkable effectiveness and propose an improved method EATNE under the network embedding scenario. Specifically, we first prove that APP can be interpreted as an optimization method. To understand this optimization method more intuitively, we further analyze the relationship between APP and other optimization methods, reaching a conclusion that APP is a Momentum-like method that enhances the previous information for updating model parameters (momentum information for short) during the training process. 
Secondly, we answer the question of why APP works well in the scene of network embedding from both theoretical analysis and experiments. We discover that APP shows extraordinary performance when the model parameters get stuck in the saturation region of the loss function. By combining the properties of network data and the relationship between APP and Momentum, we conclude that APP alleviates vanishing gradient in the saturation region that is quite common in the Sigmoid activation and verify this conclusion through empirical experiments. Finally, based on the above analysis, we propose a new method EATNE to enhance AdvTNE for a better solution to the Sigmoid saturation region problem. Experimental results on various datasets and two downstream tasks provide strong support for our theoretical analysis and the well-designed Sine-based activation of EATNE.

The main contributions of our paper can be summarized as follows:
\begin{itemize}
    \item We analyze the emerging method, adversarial perturbations on parameters,  from the perspective of optimization and reveal its relationship with the Momentum optimization method from both theory and practice.
    \item As a special case, we provide an in-depth analysis of APP on network embedding, i.e., AdvTNE.
    Through the analysis of network properties and the optimization objective as well as empirical studies, we conclude that APP enhances the momentum information in the Sigmoid saturation region and assists the training procedure so that the AdvTNE model can learn better node representations.
    \item Based on theoretical analysis, we further improve AdvTNE by designing a new Sine-based activation, which is expected to obtain a better solution than the Sigmoid activation.  
    \item We conduct extensive experiments on four real networks to validate the effectiveness of our proposed EATNE on node classification and link prediction tasks. The results demonstrate that EATNE is competitive with state-of-the-art methods.
\end{itemize}

%% file: src/model.tex
\section{Preliminaries}
\subsection{Notations} Considering a graph $G = (V, E, A)$ with  node set $V = \{v_1, ..., v_N\}$, a edge set $E = \{e_{ij}\}$, and the adjacency matrix $A$. Here $A$ is defined as $\bf{A} = [A_{ij}]$, where $A_{ij} = 1$ if $e_{ij} \in E$ and $A_{ij} = 0$ otherwise. 
$d_i = \sum_j A_{ij}$ is the degree of node $v_i$. 
Given a network $G = (V, E, A)$, the purpose of \textbf{network embedding} is to learn a function $f: V \mapsto U,$ where $U \in \mathbb{R}^{N\times r}$ with embedding dimension $r \ll N$, preserving the structural properties of $G$. We denote $\vec u_i = f(v_i)$ as the vector representation of node $v_i$.

\subsection{Skip-gram based Network Embedding}
Skip-gram based Network Embedding(SGNE), the basis of adversarial training for network embedding, is a class of popular network embedding methods, including LINE, DeepWalk, Node2vec and so on \cite{dai2019adversarial}. In these methods, a vector representation of a certain node $v_i$ is learned by maximizing its likelihood of co-appearance with its neighbor node set $N_S(v_i)$ and the loss accelerated by negative sampling is defined as follows:
\begin{equation}
\label{equ:line_objective}
\begin{split}
    L =& - \!\! \sum_{v_i \in V} \!\!\  \sum_{v_j \in N_S(v_i)} \!\!\!\!\!\!   w_{ij} \! \left(\log S_{ij}^+ \! + \! k \cdot \mathbb{E}_{v_n \sim P_{n}(v)}[\log S_{in}^-]\right)  \\
    S_{ij}^+& = \sigma(\vec u_j' \cdot \vec u_i),\qquad S_{in}^- = 1 - S_{in}^+
\end{split}
\end{equation}
where $w_{ij}$ is the weight of node pair $(v_i, v_j)$ and it can be estimated by the expectation number of co-occurrences of $(v_i, v_j)$ for models based on random walk. $\Vec{u_i}$ is the center vector of $v_i$ and $\Vec{u_j}'$ is the context vector of $v_j$. $S^+_{ij}$ is the normalized similarity of $v_i$ and $v_j$ under the corresponding representations $\Vec{u_i}$ and $\Vec{u_j}'$ learned from the model. $\sigma(\cdot) = 1 / (1 + e^{-x})$ is the classic Sigmoid activation function and $N_S(v_i)$ is the neighbor nodes set of $v_i$ with different definitions in a variety of methods. $k$ is the number of negative samples and $P_{n}(v) \propto d^\alpha_v$ is the noise distribution for negative sampling, where $d_v$ is the degree of the vertex $v$ and $\alpha$ is a hyper parameter. We set $\alpha = 1$ for simplicity, i.e., $P_n(v) \propto d_v$. 

\subsection{Adversarial Training}
The main idea of adversarial training is to augment the dataset with generated adversarial examples during the training process. Considering a typical task i.e., classification, the negative log-likelihood loss on adversarial examples is defined as follows:
\begin{equation}
\begin{split}
\label{equ:advt_max}
    \max_\theta &\quad \log p(y | x + n^{(adv)}; \theta)\\
    where &\quad n^{(adv)} = \mathop{\arg\min}_{n, ||n|| \leq \rho} \log p(y | x + n; \hat \theta)
\end{split}
\end{equation}
where $x, y$ is the pair of the input features and the label. $\theta$ denotes model parameters and $n^{(adv)}$ is the perturbation on $x$. $\rho$ acts as a norm constraint of the perturbation $n$. $\hat \theta$ represents the model parameters fixed as constants when updating $n^{(adv)}$.
To enhance the local smoothness of the learned function of the model, the data features are perturbed to the maximum extent by maximizing the loss in each iteration of optimization.
\subsection{Adversarial Training for Network Embedding and Adversarial Perturbations on Parameters}
Adversarial Training for Network Embedding (AdvTNE) generalizes the above paradigm to SGNE and utilizes Adversarial Perturbations on Parameters (APP). The overall objective can be formulated as:

\begin{equation}
\begin{split}
\label{equ:advtne}
     \min_{\Theta} \quad & L(G | \Theta) + \lambda L(G | \Theta + n^{(adv)})\\
    where & \quad n^{(adv)} = \mathop{\arg\max}_{n, ||n|| \leq \rho} L(G| \hat \Theta + n)
\end{split}
\end{equation}
where $\Theta$ represents all learnable parameters including center vector $\vec u$ and context vector $\vec u'$ and $L$ is SGNE objective defined in Eq. \eqref{equ:line_objective}.

The second term in the objective of AdvTNE directly adds adversarial perturbations in parameter space rather than input space as most adversarial models do. Besides, parameter $\Theta + n$ cannot correspond to a real node on the network $G$, raising more questions about the usage of APP. As a consequence, this new method cannot be explained by enhancing the local smoothness of the function, arousing us to find a reasonable explanation for the mechanism behind its impressive performance. 

\section{Understanding of Adversarial Perturbations on Parameters}

Our interpretation of APP will be divided into two steps. We will start with the optimization process of APP. Theoretical analysis shows that APP can be understood as an optimization method, and it demonstrates similarities with the Momentum method in certain situations. Secondly, we will analyze why AdvTNE works well as a new optimization method on the network embedding task and verify it through some experiments. Due to the page limitation, some proofs in the section are listed in the Appendix.

\begin{figure}[!ht]
    \centering
    \includegraphics[width=0.3\textwidth]{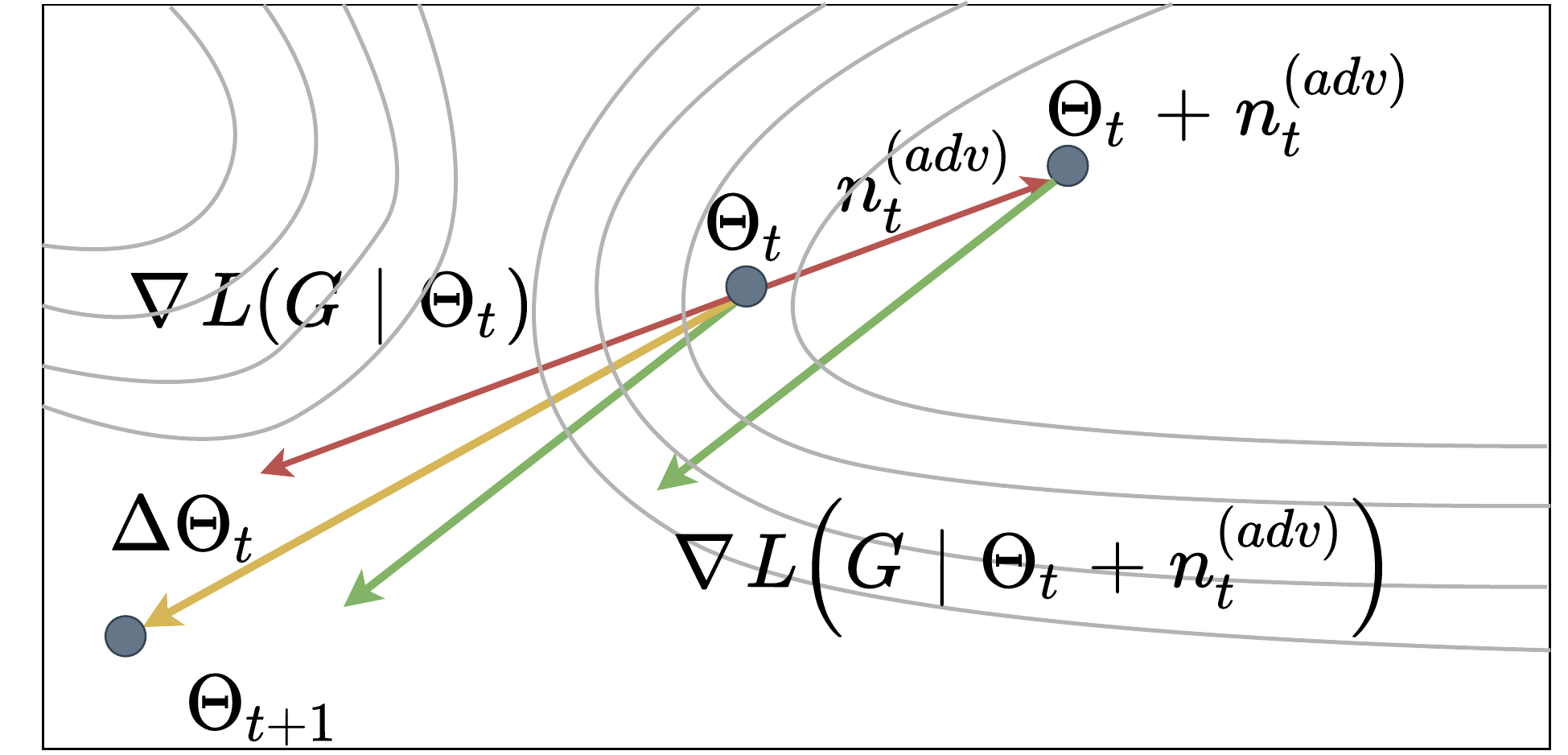}
    \caption{Update procedure of APP. The downward-pointing red line represents gradient of $\boldsymbol{L(G|\Theta)}$ and the adversarial perturbation $\boldsymbol{n_t^{(adv)}}$ is in the opposite direction. The green line is the gradient of loss at a perturbed position $\boldsymbol{\Theta_t+n_t^{(adv)}}$. The yellow line is the final descent direction.}
    \label{fig:advt_fig}
\end{figure}

\subsection{APP as New Optimization Method}
Eq. \eqref{equ:advtne} shows that APP perturbs the parameter $\Theta$ dynamically at each iteration of gradient descent. This inspire us to look into the optimization procedure of APP. As Fig.\ref{fig:advt_fig} illustrates, to optimize the overall loss function at step $t$, model parameter $\Theta_t$ will be perturbed along the perturbation direction $n_t^{(adv)}$. Then $\Theta_t$ is updated from both the gradients of $L(G | \Theta_t)$ and the perturbed objective $L(G | \Theta_t + n^{(adv)}_t)$, indicating the direction of gradient descent does not depending only on $\nabla L(G | \Theta)$. It can be viewed as an optimization strategy in essence. Therefore, we have the following theorem: 

\begin{theorem}
\label{theorem:optimization}
Given an unconstrained differentiable loss function $L(\Theta)$, applying APP strategy (Eq. \eqref{equ:advtne}) with gradient descent is equivalent to iterative optimization based on the following strategy when updating the parameter $\Theta$:

\begin{equation}
\begin{split}
\label{equ:advtne_strategy}
    \Theta_{t+1} &= \Theta_t - \epsilon \Delta \Theta_t,\\
    \Delta \Theta_t =  \nabla & L(\Theta_t) + \lambda \nabla L(\Theta_t + n^{(adv)}_t)
\end{split}
\end{equation}

where $\Theta_t$ is the value of $\Theta$ in the $t$-th iteration, $\epsilon$ is the learning rate, and $n^{(adv)}_t$ is the adversarial perturbation which equals $\mathop{\arg\max}_{n, ||n|| \leq \rho} L(\Theta_t + n)$.
\end{theorem}

Eq. \eqref{equ:advtne_strategy} holds for every reasonable loss $L$, and it provides the way to update parameter $\Theta$ so that APP can indeed be regarded as an optimization method. Note that APP updates $\Theta$ and $n^{(adv)}$ every iteration following common used manner proposed in \cite{goodfellow2014explaining}.
\subsection{Property of the APP Optimization Method}
In order to have a deeper understanding of this optimization method, we analyze its property especially when the loss function gets stuck in the saturation region or is about to converge. We discover that the current value of perturbed loss $L(\Theta_t +n)$ is very close to the value of loss at last step, i.e., $L(\Theta_{t-1})$. Mathematically, we have the following theorem:

\begin{theorem}
\label{theorem:advt_approx}
If (1) $L$ is L-smooth, i.e., $||\nabla L (\Theta_t) - \nabla L(\Theta_{t-1})||\leqslant l ||\Theta_t - \Theta_{t-1}||$ 
and (2) $||\nabla L(\Theta_t)||\leqslant\delta$ hold where $l,\delta$ are both finite positive constants, then we have the following inequality:
\begin{equation}
\label{eq:max}
    \left|\max_{||n||\leq\rho}\{L(\Theta_t+n)\}-L(\Theta_{t-1}) \right| \leqslant (1+ \lambda \epsilon l)\rho \delta + \epsilon(1+\lambda) \delta^2.
\end{equation}
\end{theorem}

Based on Theorem. \ref{theorem:advt_approx}, when the objective gets stuck in the saturation region or is about to converge, we have the following two corollaries:
\begin{customthmc}{2.1}
\label{cor1}
The perturbation in Eq. \eqref{equ:advtne} can be approximated as:
\begin{equation}
    n^{(adv)}_t = \Theta_{t-1} - \Theta_t.
\end{equation}
\end{customthmc}
\begin{customthmc}{2.2}
\label{cor2}
The AdvTNE optimization method can be formulated as:
\begin{equation}\label{eq:advt_momentum}
\begin{split}
    \Theta_{t+1} &= \Theta_t- \epsilon \Delta\Theta_t,\\
    \Delta \Theta_t = \nabla &L(\Theta_t) +\lambda \nabla L(\Theta_{t-1}).
\end{split}
\end{equation}
\end{customthmc}

Theorem. \ref{theorem:advt_approx} and two corollaries demonstrate that the optimal adversarial perturbation $n_t^{(adv)}$ can be approximated as $\Theta_{t-1} - \Theta_t$ when the objective gets stuck in the saturation region or is about to converge, in other words, update procedure of APP can be approximated as a combination of gradient of loss at current step $t$ and the last step $t-1$. As the original method requires additional calculation of $\nabla L(\Theta_t+n_t^{(adv)})$, the training procedure will be speed up if we directly replace it with $\nabla L(\Theta_{t-1})$. Theorem. \ref{theorem:advt_approx} provides the theoretical justification for this approximation.

We notice that Eq. \eqref{eq:advt_momentum} is similar to the formula of other well-known optimization methods, and this arouses us to investigate their relations in the next section.

\subsection{Relationship with Momentum}
To understand the new optimization method more intuitively, we further analyze the relationship between the APP strategy and classic optimization methods like Momentum \cite{ruder2016overview}. 
We first show the update strategy of Momentum:
\begin{align}
    \Theta_{t+1} &= \Theta_t - \epsilon \Delta \Theta_t,\\ 
    \label{equ:mom_update}
    \Delta  \Theta_t = \nabla & L(\Theta_t) + \eta \Delta \Theta_{t-1} 
\end{align}
where $\eta \in (0, 1)$ is a hyper-parameter controlling the weight of historical information. 

In particular, Momentum computes Exponential Moving Average of the gradient sequence to run through the saturation region, which can be seen by expanding Eq. \eqref{equ:mom_update} as:
\begin{align}
    \Delta  \Theta_t &= \nabla L(\Theta_t) + \eta \nabla L(\Theta_{t-1}) + \eta^2 \nabla L(\Theta_{t-2}) + \cdots \notag \\
    \label{equ:mom_apr}
    &= \nabla L(\Theta_t) + \eta \nabla L(\Theta_{t-1}) + \mathcal{O}(\eta^2)
    \approx \nabla L(\Theta_t) + \eta \nabla L(\Theta_{t-1}).
\end{align}
Eq. \eqref{equ:mom_apr} indicates that weights of gradients at previous states will decrease exponentially by $\eta$. Comparing the updating strategies of Momentum in Eq. \eqref{equ:mom_apr} and APP in Eq. \eqref{eq:advt_momentum}, it is apparent that they are similar in the parameter update procedure, especially if $\mathcal{O}(\eta^2)$ is small enough and can be ignored. Intuitively, when the loss is about to converge or gets stuck in the saturation region, the optimal adversarial direction functions as vanishing the effect of updating; that is to say, the parameter at step $t$ is perturbed back to the last step $t-1$. Consequently, the adversarial term becomes a momentum term.

Furthermore, we conduct experiments to verify this conclusion. We compare the performance of DeepWalk optimized with Momentum and APP (i.e., AdvTNE), respectively. Experiment settings are described in Section. \ref{sec:exp} and the results are presented on Tab. \ref{tab:momentum}. We can find that utilizing APP to optimize DeepWalk achieves comparable results with Momentum-optimized DeepWalk, indicating that APP is similar to Momentum from both theory and practice. As APP performs well in network embedding, we can attribute the effectiveness to the momentum information. 

We also know that APP and Momentum are similar especially when the loss function gets stuck in the saturation region or the model parameters are about to converge. As models utilizing APP like AdvTNE achieve impressive performance improvement compared with the original ones, we believe that the APP method shows its superiority over the original method in this scene. However, it raises another question: is the phenomenon of saturation region severe on network embedding? We will discuss it in the next section.

\subsection{Why APP Performs Well on Network Embedding}
APP, regarded as a new optimization method, decreases the loss function in a way that strengthens the momentum information, especially when the loss function is on the saturation region. However, we still have no idea why the model performance is improved when the optimization method emphasizes more on this region. In this section, we will reveal the saturation region problem on network embedding if Sigmoid is selected as the activation. Note that AdvTNE can be utilized for all Skip-gram based Network Embedding (SGNE) methods, and we take LINE as an example in the following analysis for simplicity.

\paragraph{Intuition} We directly post the intuitive conclusion of our theoretical analysis. 
In a scale-free network, the similarity $S^+$ of most node pairs learned by SGNE methods approaches 1, leading to indistinguishable embedding vectors of these nodes. It increases the burden on downstream tasks as it will be more challenging to distinguish these nodes or node pairs precisely. We expect to accurately learn nodes representations that can enable the downstream learner to tell slight differences between the similarity of node pairs. 
However, utilizing Sigmoid as the activation in SGNE models will force $S^+$ of node pairs with high similarity value to fall into the saturation region of the Sigmoid function. The gradients of parameters are about to vanish so that the similarity cannot be optimized anymore. Considering the APP method, an important reason it works well is that it collects the momentum information to enhance the gradient in the Sigmoid saturated region and is able to reach a better minimum.

To verify our intuition, we analyze the theoretical optimal solution of SGNE method and calculate its value in a practical situation. First, we have the following conclusion:

\begin{theorem}
\label{theorem:minimum}
For sufficiently large embedding dimensionality $r$, Eq. \eqref{equ:line_objective} The minimum for similarity of point pair $(v_i, v_j)$ has the following form:
\begin{equation}
\begin{split}
\label{equ:line_converge}
    S_{ij}^+ = \frac{w_{ij}}{w_{ij}+(\frac{d_id_j}{D})k},
\end{split}
\end{equation}
where $d_i = \sum_j w_{ij}$ and $D = \sum_i d_i$.
\end{theorem}

 Note that it is not so difficult to meet the condition that the embedding dimensionality should be sufficiently large. As is discussed in \cite{levy2014neural}, $r\sim 100$ can also achieve satisfactory performance while satisfying $r\ll N$. The degree of nodes in a scale-free network follows a power-law distribution, meaning the degree of most nodes is very small. The number of negative sample $k$ is usually set as $3\sim10$ in practice. $(\frac{d_id_j}{D})k \to 0$ holds for linked node pairs $(w_{ij} > 0)$, so that $S_{ij}^+ \to 1$. It means that the similarity of linked nodes will all approximate 1. In order to have a more distinct 
comprehension of this result, we proved the following theorem:

\begin{theorem}
Considering a scale-free network that the degree $d$ of each node follows the power-law as $d\sim P(d) \propto d^{-\alpha}, \, \alpha \in (2, 3)$, the number of edges is $|E|$ and the number of negative samples is $k$. Given a threshold $\gamma \in (0,1)$, the following equation holds:
\begin{equation}
\label{equ:s}
    P(S^+_{ij} \geq \gamma) \approx \frac{(\alpha-1)^2}{{-\alpha+2}}\frac{\log |E|}{2|E|}\cdot(R^{\alpha-2}-1),
\end{equation}
where $R=\frac{\gamma k}{2(1-\gamma)|E|}$.
\end{theorem}

The proof can be found in the Appendix. According to Eq. \eqref{equ:s}, we can calculate the expected ratio of linked nodes that their similarity $S^+$ is higher than a threshold $\gamma$ to the total number of edges:

\begin{equation}
    \mathbb{E}[\frac{\#(S^+ > \gamma)}{|E|}] = \frac{|V|(|V| - 1)P(S^+ > \gamma)}{2|E|}.
\end{equation}

Here is a specific example. Let $|V| = 1000$, $|E| = 3000$, $k = 4$, $\alpha = 2.5$, 
$\mathbb{E}[\frac{\#(S^+ > 0.9)}{|E|}]$ equals to $81.6\%$. It can be seen that the number of node pairs whose similarity falls in the Sigmoid saturation region accounts for a sufficiently large proportion. The problem of the Sigmoid saturation region is extremely severe to damage the quality of node embedding. Luckily, the APP method, which enhances its optimization power by incorporating momentum information, can achieve a remarkable result under this situation. Therefore, we believe it is one of the reasons that APP performs so well.

\begin{figure}[!ht]
    \centering
    \includegraphics[width=0.5\textwidth]{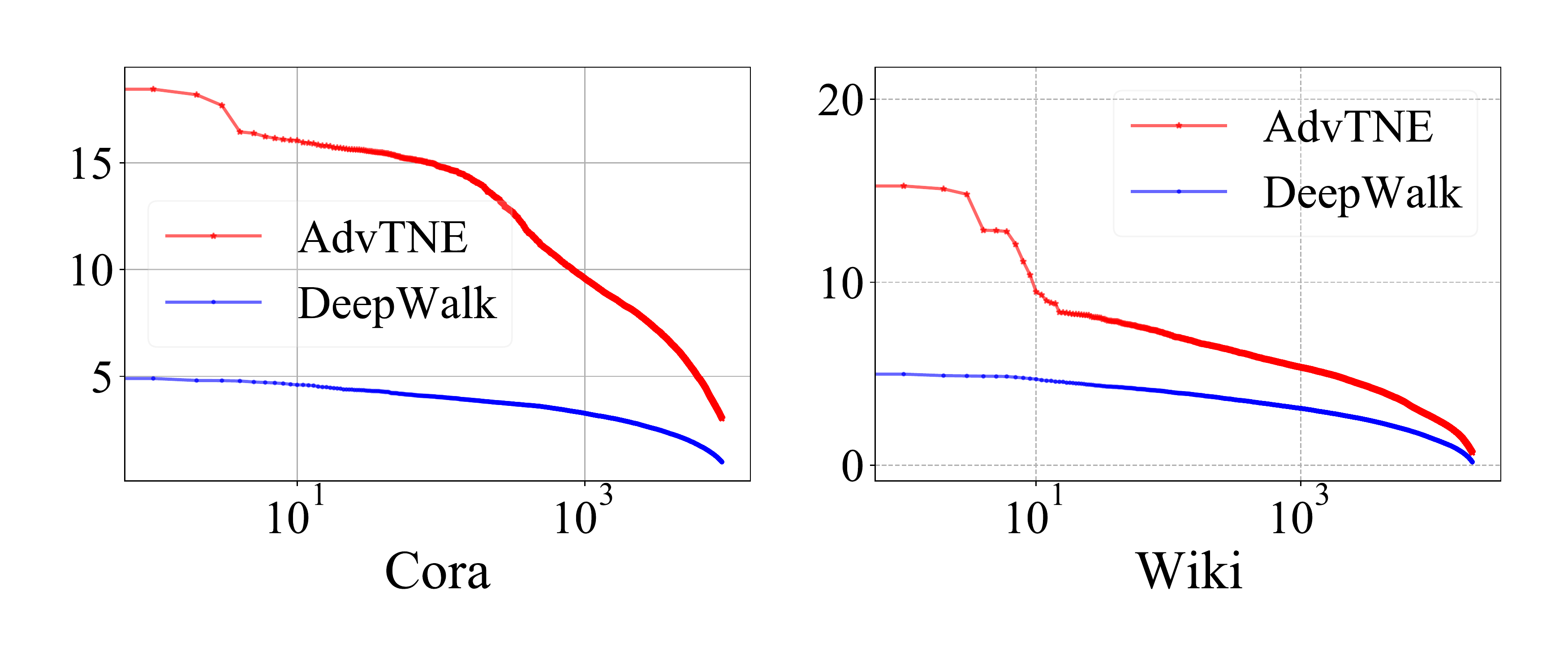}
    \caption{Comparison between DeepWalk and the corresponding AdvTNE method. The higher line represents high similarity values calculated by the node embeddings.}
    \label{fig:ppmi}
\end{figure}

To verify our assumption, we make a comparison between the original DeepWalk model and AdvTNE taking DeepWalk as the base. We follow the best experiment setting and more details are elaborated in Section. \ref{sec:exp}. 
For simplicity, we calculate the similarity of linked nodes based on their normalized representation vectors. The simplified similarity is $\tilde u_j' \cdot \tilde u_i$ where $\tilde u_j'$ and $\tilde u_i$ are normalized vectors of $\vec u_j'$ and $\vec u_i$, respectively.
We select 85\% node pairs that scores highest in the corresponding PPMI matrix and depict their simplified similarity. Node pairs with higher PPMI values are supposed to have higher simplified similarity scores. The results are displayed in Fig. \ref{fig:ppmi} and the simplified similarity values are sorted by their PPMI values for visualization.
As Figure. \ref{fig:ppmi} illustrates, on the paper citation network Cora and the page-page network Wiki, the similarity values of the node pair whose embeddings are learned by AdvTNE are much higher, indicating that the APP can indeed help the base model learn higher similarities of node pairs, thus promote the model
performance.

\section{Another Road to Better Performance}
Through the above analysis, we can conclude that (1) similarity of connected node pairs whose embeddings are learned through SGNE will concentrate on high values;
(2) the Sigmoid-based loss function (i.e., loss function with Sigmoid as the activation $\sigma(\cdot)$) in SGNE has an evident saturation region effect, which increases the difficulty of convergence for the node representations. Thus, the performance in downstream tasks is substantially affected by the imperfect node representations. 
(3) APP solves this problem from the perspective of \textbf{improving the optimization process.} Intuitively, the adversarial perturbations provide the ``inertial'' component that contributes to a better descent and alleviates being stuck in the saturation region. 

It also inspires us to find a more direct way to solve the problem of node similarities falling in the saturation region. In other words, since the Sigmoid-based loss function brings an evident saturation region effect, \textbf{can we seek its substitutes for reducing this effect and improve the quality of embedding vectors?} 
The substitute should satisfy some restrictions: (1) the activation must be bounded to guarantee the model convergence; (2) its saturation region should be small enough to enable an easier optimization to the optimal value. 

Since (1) and (2) are hard to satisfy simultaneously for common-used activation, we turn to periodic functions that have considerably small saturation regions. Besides, periodic functions having many equal optimal values will reduce the optimization difficulty as the parameters can be optimized to any one of them.
Motivated by this idea, we propose a new activation to \textbf{E}nhance the \textbf{A}dv\textbf{TNE} and we name it \textbf{EATNE}. The activation can be formulated as:
\begin{equation}
    \label{eq:sine-based}
    \begin{split}
        &T_{ij}^+ = \frac{1}{2}\left(1+\sin[\frac{\pi}{2} W_T(\Vec{u}'_j\otimes \Vec{u}_i)]\right)+\delta \\
        &T_{in}^- = \frac{1}{2}\left(1-\sin[\frac{\pi}{2} W_T(\Vec{u}'_n\otimes \Vec{u}_i)]\right)+\delta \\
        L(G|\Theta) =&- \!\! \sum_{v_i \in V} \!\! \sum_{v_j \in N_S(v_i)} \!\! w_{ij}\left(\log T_{ij}^+ \! + \! k \cdot \mathbb{E}_{v_n \sim P_{n}(v)}[\log T_{in}^-]\right).
    \end{split}
\end{equation}
Here, we denote the node similarity between node $v_i$ and $v_j$ in EATNE as $T_{ij}^+$. $W_T \in \mathbb{R}^{1\times r}$ is a learnable parameter and $\otimes$ denotes the element-wise multiplication between the embeddings of two nodes. $\delta > 0$ is a small constant to ensure that $T_{ij}^+$ and $T_{in}^-$ is greater than 0 as $\log 0 \to -\infty $. Inspired by \cite{siren}, we apply batch normalization on the term $\Vec{u}'_j\otimes \Vec{u}_i$ and $\Vec{u}'_n\otimes \Vec{u}_i$ so that they follows the uniform distribution $\mathcal{U}(-1,1)$. $W_T$ is also initilized as $W_T\sim \mathcal{U}(-\sqrt{6/r},\sqrt{6/r})$ to ensure that $W_T(\Vec{u}'_j\otimes \Vec{u}_i)\sim \mathcal{N}(0,1)$ and $W_T(\Vec{u}'_n\otimes \Vec{u}_i)\sim \mathcal{N}(0,1)$. Here $\mathcal{N}(\cdot,\cdot)$ is the normal distribution. Eq. \eqref{eq:sine-based} can be viewed as a \textbf{Sine-based loss function}.

It is not hard to verify that this activation satisfies the first restriction. To examine the second restriction with this activation and to show the superiority of the Sine-based loss function, we compare the learning progress with the Sigmoid-based one. With Sigmoid-based loss function, the optimal value can be reached by
\begin{equation}
\begin{split}
    &S^+_{ij} = \frac{1}{1+\exp(-\Vec{u}'_j\cdot\Vec{u}_i)}=1 \Rightarrow \Vec{u}'_j\cdot\Vec{u}_i\to\infty.
\end{split}
\end{equation}

For Sine-based loss function, the optimal value is 
\begin{equation}
\begin{split}
    \frac{1}{2}&\left(1+\sin[\frac{\pi}{2} W_T(\Vec{u}'_j\otimes \Vec{u}_i)]\right)+\delta = 1\Rightarrow\\
    &W_T(\Vec{u}'_j\otimes\Vec{u}_i)=\frac{2}{\pi }\arcsin(1-2\delta).
\end{split}
\end{equation}
In practice, $\Vec{u}'_j\cdot\Vec{u}_i$ is hard to reach $\infty$ for computer precision. In addition, when $\Vec{u}'_j\cdot\Vec{u}_i \to \infty$, the norm of $\Delta\Theta$ meets the expression that $||\Delta\Theta|| \sim S^+_{ij}(1-S^+_{ij})\approx 0$, indicating that the parameters almost cannot be updated after certain iteration. However, the range of arcSine function is finite, enabling the fast convergence of the node representation vectors. Besides, when $T^+_{ij} \to 1$, it holds that $||\Delta \Theta|| \sim \sqrt{\delta(1-\delta)}$, indicating the norm of $\Delta\Theta$ is much bigger than that of Sigmoid-based loss function. In conclusion, our Sine-base loss function satisfies the two necessary 
restrictions and it is supposed to be effective on network embedding.

\begin{table}[!ht]
\small
\centering
\caption{Statistics of benchmark datasets}
\label{tab:dataset}
\begin{tabular}{c|cccc}
\toprule
Datasets & \begin{tabular}[c]{@{}c@{}} \# Nodes\end{tabular} & \begin{tabular}[c]{@{}c@{}}\# Edges\end{tabular} & \begin{tabular}[c]{@{}c@{}}\# Avg.Degree\end{tabular} & \begin{tabular}[c]{@{}c@{}}\# Labels\end{tabular} \\ \midrule
Cora & 2708 & 5278 & 1.95 & 7 \\
Citeseer & 3264 & 4551 & 1.39 & 6 \\
Wiki & 2363 & 11596 & 4.91 & 17 \\
PubMed & 19717 & 88648 & 4.50 & 3 \\ 
\bottomrule
\end{tabular}
\end{table}
\section{Experiment}
\label{sec:exp}
We conduct experiments on real-world networks on node classification and link prediction task to verify our analysis: 
\begin{itemize}
    \item \textbf{Q1}: Can APP achieve results comparable to Momentum optimization if APP and Momentum are similar in the formula?
    \item \textbf{Q2}: How does our new activation perform compared to the state-of-the-art network embedding models?
\end{itemize}

\subsection{Experiment Setup} 
\subsubsection{Datasets}
We conduct experiments on four real-world datasets. 
Cora, Citeseer, and PubMed \cite{cora} are all paper citation networks where nodes correspond to paper and edges represent citation links. Wiki \cite{wiki} is a network with many web pages as its nodes and hyperlinks between web pages as its edges. We delete nodes with self-loops and zero degrees on original datasets. Statistics of these datasets are summarized in Tab. \ref{tab:dataset}.

\begin{table*}[!ht]
\centering
\caption{Comparison of AS and NM on Cora, Wiki and Citeseer, revealing the similarity of APP and Momentum in practice.}
\label{tab:momentum}
\resizebox{\textwidth}{!}{
\begin{tabular}{c|c|c|c|c|c|c|c|c|c|c|c|c|c|c|c|c|c|c|c}
\toprule

Ratio        & \textbf{Mean}  & 1\%            & 2\%   & 3\%   & 4\%   & 5\%   & 6\%   & 7\%   & 8\%   & 9\%   & 10\%  & 20\%  & 30\%  & 40\%  & 50\%  & 60\%  & 70\%  & 80\%  & 90\%  \\ \midrule
Cora\_AS     & 74.48 & 55.06          & 62.64 & 65.38 & 67.83 & 69.54 & 69.51 & 71.22 & 71.99 & 73.24 & 74.78 & 78.21 & 80.27 & 81.84 & 82.77 & 83.38 & 83.51 & 84.06 & 85.42 \\
Cora\_NM     & 76.32 & 63.32          & 67.81 & 69.20 & 71.15 & 73.37 & 73.55 & 74.51 & 75.36 & 75.54 & 76.35 & 79.88 & 80.79 & 81.47 & 81.74 & 82.05 & 81.91 & 82.32 & 83.51 \\ \midrule
Wiki\_AS     & 56.71 & 34.19          & 44.69 & 48.18 & 51.02 & 51.77 & 51.96 & 54.37 & 54.62 & 55.77 & 56.36 & 60.21 & 62.71 & 63.65 & 64.80 & 65.64 & 66.04 & 67.34 & 67.43 \\
Wiki\_NM     & 56.87 & 36.25          & 46.56 & 49.59 & 52.25 & 52.57 & 54.70 & 55.39 & 55.05 & 56.68 & 57.01 & 59.76 & 62.35 & 62.93 & 63.42 & 64.22 & 64.37 & 64.93 & 65.57 \\ \midrule
Citeseer\_AS & 51.47 & 36.45 & 40.45 & 42.76 & 43.78 & 46.28 & 47.31 & 48.00 & 48.54 & 50.55 & 50.97 & 55.11 & 56.86 & 58.29 & 58.71 & 59.55 & 60.83 & 60.32 & 61.68  \\
Citeseer\_NM & 51.56 & 36.28          & 40.43 & 43.38 & 44.46 & 46.88 & 47.36 & 48.74 & 49.02 & 50.71 & 50.79 & 55.07 & 57.47 & 58.41 & 58.93 & 59.82 & 60.59 & 59.98 & 59.69 \\ 
\bottomrule
\end{tabular}%
}
\end{table*}

\begin{table*}[!h]
\centering
\caption{Accuracy(\%) of multi-class classification on Cora}
\label{cora}
\resizebox{\textwidth}{!}{%
\begin{tabular}{c|c|c|c|c|c|c|c|c|c|c|c|c|c|c|c|c|c|c|c}
\toprule

Ratio & \textbf{Mean} &1\%   & 2\% & 3\% & 4\% & 5\% & 6\% & 7\% & 8\% & 9\% & 10\% & 20\% & 30\% & 40\% & 50\% & 60\% & 70\% & 80\% & 90\% \\ \midrule
DeepWalk    & 74.55   & 59.57 & 63.54  & 66.75  & 69.43  & 70.54  & 70.82  & 71.77  & 72.66  & 73.60  & 74.43   & 77.75   & 79.30   & 80.50   & 80.89   & 82.31   & 82.12   & 82.32   & 83.58   \\
LINE        & 66.02  & 50.93 & 56.16  & 60.55  & 62.14  & 63.89  & 63.46  & 65.09  & 65.07  & 66.18  & 66.53   & 68.73   & 70.15   & 70.31   & 71.07   & 70.93   & 72.20   & 72.16   & 72.80   \\
node2vec    & 74.71  & 59.18 & 64.03  & 66.94  & 69.20  & 70.57  & 71.00  & 71.45  & 72.60  & 73.33  & 74.16   & 77.57   & 79.93   & 80.84   & 81.82   & 82.14   & 82.85   & 83.03   & 84.13   \\
GF       & 41.35 & 20.79 & 29.76 &    33.65 &    34.68 &    36.13 &    37.65 &    39.01 &    38.93 &    40.28 &    40.53 &    44.85 &    46.80 &    48.45 &    49.75 &    49.63 &    50.23 &    51.75 &    51.40  \\
GraRep   &71.45 & 55.01 &    60.85 &    66.12 &    68.09 &    67.26 &    70.77 &    70.81 &    71.27 &     71.47 &    72.73 &     74.42 &  75.68 & 75.70 &    76.13 & 78.13 &     77.28 &     77.43 &     76.97 \\

GraphSage     &53.62 & 31.51 &     38.79 &    44.29 & 45.07 &    47.29 &    49.77 &     50.43 &     51.77 &    53.92 &     54.18 &     58.24 &     59.59 &     62.07 &     62.89 &     63.55 &     63.79 &     62.88 &     65.09  \\

AdvTNE      &75.14    & 58.95 & 65.18 & 67.21 & 68.62 & 71.19 & 71.45 & 72.41 & 73.39 & 73.81 & 75.15 & 77.88 & 79.94 & 81.03 & 82.22 & \textbf{82.97} & \textbf{83.11} & \textbf{83.71} & \textbf{84.24} 
  \\


Cleora & 70.99 & 54.29 & 60.97 & 62.62 & 65.68 & 68.51 & 68.65 & 68.82 & 71.17 & 71.04 & 72.08 & 73.85 & 75.62 & 76.57 & 76.82 & 76.98 & 77.48 & 77.90 & 78.75 \\
  
  \midrule
EATNE & \textbf{76.93}&\textbf{60.30} & \textbf{66.28} & \textbf{70.56} & \textbf{72.43} & \textbf{74.96} & \textbf{74.60} & \textbf{75.44} & \textbf{76.64} & \textbf{77.33} & \textbf{77.55} & \textbf{80.12} & \textbf{81.09} & \textbf{81.57} & \textbf{82.47} & 82.90   & 83.06   & 83.47   & 84.02   \\

\bottomrule
\end{tabular}
}
\end{table*}

\begin{table*}[!h]
\centering
\caption{Accuracy(\%) of multi-class classification on Citeseer}
\label{Citeseer}
\resizebox{\textwidth}{!}{%
\begin{tabular}{c|c|c|c|c|c|c|c|c|c|c|c|c|c|c|c|c|c|c|c}
\toprule

Ratio & \textbf{Mean} & 1\%   & 2\% & 3\% & 4\% & 5\% & 6\% & 7\% & 8\% & 9\% & 10\% & 20\% & 30\% & 40\% & 50\% & 60\% & 70\% & 80\% & 90\% \\ \midrule
Deepwalk & 51.76 &	38.53 &	41.53 &	43.63 &	44.41 &	46.62 &	47.25 &	48.35 &	48.63 &	49.59 &	50.65 &	55.06 &	57.46 &	58.15 &	59.15 &	59.43 &	61.21 &	60.46 &	61.50 \\
LINE          & 41.94 & 35.19 & 36.55  & 39.06  & 38.49  & 40.34  & 40.95  & 41.33  & 40.79  & 41.80  & 41.87   & 42.99   & 43.77   & 44.40   & 45.20   & 44.98   & 45.19   & 45.91   & 46.09   \\
node2vec      & 52.08 & 39.62 & 42.19  & 44.54  & 45.62  & 46.92  & 48.17  & 48.63  & 49.36  & 50.27  & 50.91   & 55.08   & 57.53   & 58.55   & 59.29   & 59.77   & 60.54   & 60.28   & 60.21   \\
GF            & 30.21 &  21.00 &  23.22 &  24.54 &  25.96 &  26.20 &  27.61 &  28.07 &  28.01 &  28.22 &  29.12 &  31.81 &  33.80 &  34.92 &  35.00 &  35.54 &  35.67 &  37.55 &  37.61 \\       
GraRep        & 48.17 &  35.37 &  40.10 &  41.78 &  44.79 &  45.59 &  46.20 &  47.93 &  47.38 &  47.72 &  47.87 &  50.39 &  51.88 &  52.05 &  53.20 &  53.00 &  54.00 &  54.80 &  53.08 \\
GraphSage     & 30.97 & 20.94 &  23.54 &  24.04 &  25.93 &  26.41 &  26.96 &  28.62 &  28.29 &  29.58 &  30.52 &  32.97 &  35.40 &  35.60 &  37.43 &  36.88 &  37.84 &  37.43 &  39.02 \\
AdvTNE        & 53.14 & 39.02 & 41.94 & 44.77 & 45.21 & 47.43 & 49.06 & 49.72 & 51.09 & 51.61 & 52.72 & 57.06 & 59.56 & 60.14 & 60.64 & 61.30 & 61.87 & 61.23 & 62.14 
   \\ 
Cleora & 52.85 &	\textbf{42.14} &	44.10 &	\textbf{48.19} &	\textbf{48.67} &	50.39 &	51.24 &	51.32 &	52.61 &	52.60 &	52.93 &	54.71 &	55.97 &	56.17 &	57.04 &	57.50 &	58.67 &	58.18 &	58.93 \\	
   
   \midrule
EATNE         & \textbf{54.87} & 38.10  & \textbf{44.37} & 47.15 & 48.62 & \textbf{50.98} & \textbf{51.95} & \textbf{53.22} & \textbf{53.44} & \textbf{54.55} & \textbf{54.82} & \textbf{58.05} & \textbf{59.91} & \textbf{60.63} & \textbf{60.96} & \textbf{61.85} & \textbf{62.66} & \textbf{62.89} & \textbf{63.46} \\

\bottomrule
\end{tabular}
}
\end{table*}

\begin{table*}[!h]
\centering
\caption{Accuracy(\%) of multi-class classification on Wiki}
\label{wiki}
\resizebox{\textwidth}{!}{
\begin{tabular}{c|c|c|c|c|c|c|c|c|c|c|c|c|c|c|c|c|c|c|c}
\toprule

Ratio & \textbf{Mean} & 1\%   & 2\% & 3\% & 4\% & 5\% & 6\% & 7\% & 8\% & 9\% & 10\% & 20\% & 30\% & 40\% & 50\% & 60\% & 70\% & 80\% & 90\% \\ \midrule
DeepWalk   & 57.81 &
34.14 &	46.01 &	49.35 &	52.02 &	52.34 &	54.51 &	55.41 &	56.30 &	57.67 &	58.69 &	61.85 &	63.49 &	64.65 &	65.52 &	66.06 &	66.29 &	67.67 &	68.69 \\

LINE          & 48.59 &30.67 & 39.73  & 42.58  & 45.05  & 46.00  & 46.89  & 48.18  & 47.64  & 48.81  & 49.09   & 51.71   & 52.88   & 53.03   & 53.61   & 53.82   & 54.40   & 54.44   & 56.16   \\
node2vec      & 57.47 &35.40 & 45.76  & 49.20  & 51.13  & 51.10  & 53.87  & 55.46  & 55.55  & 57.14  & 57.58   & 60.71   & 63.54   & 64.35   & 65.25   & 65.97   & 66.21   & 67.55   & 68.69   \\
GF            & 41.84&20.37 &  26.45 &  30.48 &  34.22 &  37.21 &  37.69 &  38.09 &  40.82 &  41.25 &  41.71 &  46.42 &  48.65 &  49.97 &  50.80 &  51.67 &  53.03 &  52.01 &  52.28   \\
GraRep        & 57.79&32.57 &  44.08 &  49.00 &  53.62 &  55.26 &  55.65 &  56.01 &  57.90 &  58.11 &  59.00 &  62.17 &  64.20 &  64.64 &  65.00 &  65.41 &  65.77 &  65.64 &  66.16   \\
GraphSage     & 46.91&25.38 &  33.20 &  36.36 &  39.86 &  40.47 &  43.26 &  45.15 &  45.59 &  46.51 &  47.10 &  51.04 &  53.61 &  54.38 &  54.48 &  56.21 &  56.73 &  57.32 &  57.72  \\
AdvTNE & 57.83 & \textbf{35.56} &45.89 &49.35 &52.15 &51.67 &53.75 &54.77 &55.66 &56.52 &56.85 &61.29 &63.57 &65.50 &66.12 &66.72 &67.31 &\textbf{69.15} &69.07\\
Cleora & 56.66 &	33.11&	44.08&	49.42&	50.59&	51.65&	53.77&	56.09&	55.52&	57.02&	57.78&	60.10&	62.12&	63.19&	63.59&	63.66&	65.42&	66.49&	66.24 \\
\midrule
EATNE         & \textbf{59.14}&34.69 & \textbf{47.13} & \textbf{50.93} & \textbf{53.80} & \textbf{53.41} & \textbf{56.08} & \textbf{57.83} & \textbf{57.29} & \textbf{58.81} & \textbf{59.37} & \textbf{62.36} & \textbf{64.99} & \textbf{65.98} & \textbf{66.94} & \textbf{67.29} & \textbf{67.86} & 68.56 & \textbf{71.18} \\
\bottomrule
\end{tabular}%
}
\end{table*}

\subsubsection{Baseline Methods} 
We compare EATNE with the following baseline models: 

\begin{itemize}
    \item Graph Factorization (GF) \cite{gf}: GF utilizes a stochastic gradient descent technique to factorizes the adjacency matrix, and it can scale to large networks.
    \item DeepWalk \cite{Perozzi2014DeepWalk}: DeepWalk obtains node embedding by applying the skip-gram model on node sequences sampled from the truncated random walk.
    \item LINE \cite{Tang2015LINE}: LINE takes into consideration node co-occurrence probability and node conditional probability to preserve network structural proximities. The problem of expensive computation is also alleviated through the negative sampling approach.
    \item node2vec \cite{grover2016node2vec}: node2vec samples node sequences with a more flexible method and balances the local and global structural properties.
    \item GraRep \cite{Cao2015GraRep}: GraRep learns node embeddings by using the SVD technique to different k-step probability transition matrix and concatenating all k-step representations.
    \item GraphSage \cite{hamilton2017inductive}: GraphSage collects neighborhood information by constructing a local computing graph and inductively learns node embeddings.
    
    \item AdvTNE \cite{dai2019adversarial}: AdvT adopts adversarial training methods on network embedding by adding perturbation to embedding, which can be treated as a regularization term.
    \item Cleora \cite{rychalska2021cleora}: Cleora is a efficient unsupervised network embedding method that does not optimize an explicit objective and sample positive or negative examples. 
\end{itemize}

\subsubsection{Experiment Settings}

We apply AdvTNE and EATNE on DeepWalk and the walk length, walks per node, window size, negative size, regularization strength and batch size are set to 40, 1, 5, 5, 1 and 2048, respectively. The dimension of embedding vectors is 128 and we training each model for 100 epochs. We utilize the recommended hyper-parameters and optimization methods of baseline models mentioned in the corresponding paper. Random search \cite{rs} is applied to find the optimal adversarial noise level $\epsilon$, regularization factor $\lambda$ and learning rate.
 The search range for $\epsilon$ is 0.1$\sim$5. We set 1e-5$\sim$1 as the search range for learning rate and 1e-2$\sim$1e3 for the regularization factor $\lambda$. 

\subsubsection{Downstream Task}
We select link prediction and node classification as the downstream tasks to evaluate the quality of representation vectors. The experiment results are averaged over ten runs on both tasks. In link prediction, $80\%$ of the edges are randomly sampled as the positive training samples with the same amount of negative training samples, i.e., node pair without direct edge. The rest $20\%$ edges together with two times of negative samples are sampled to construct the test set. AUC is adopted to measure the performance of link prediction as we train an $L_2$-SVM classifier to predict links with the edge features like other works \cite{dai2019adversarial}. Here we use the Hadamard product of representation vectors from the endpoints of the same edge as the edge features. As for node classification, 18 training ratios, ranging from $1\%\sim9\%$ and $10\%\sim90\%$, are set to train the support vector classifier. It is designed to demonstrate the ability under different proportions of training data so we utilize accuracy as the metric.

\subsection{Similarity of APP and Momentum (Q1)}

To support our claim that AdvTNE is a Momentum-like optimization method,
we compare two models on the node classification task: AdvtNE optimized with SGD (AS) and Non-AdvTNE optimized with Momentum (NM, i.e., DeepWalk only). The experiments are conducted on Cora, Wiki and Citeseer. The adversarial noise level $\epsilon$ is 0.9 for Cora, 1.1 for Citeseer and 0.6 for Wiki, respectively. We set 0.001 as learning rate and 1.0 as the regularization factor $\lambda$ for AS following the original paper. For NM, learning rate is 0.1040 for Wiki, 0.0276 for Citeseer and 0.4524 for Cora. As shown in Table. \ref{tab:momentum}, NM can achieve comparable results compared with AS. When AdvTNE is optimized by SGD, based on Theorem. \ref{theorem:optimization}, it can be viewed as using APP to optimize DeepWalk, revealing that APP is indeed similar to Momentum from these results. Note that original AdvTNE is optimized with Adam and its performance is slightly higher. We can also discover that NM surpasses AS in the low training ratio case, demonstrating the strong generalization ability of Momentum in the semi-supervised learning scenario.

\subsection{Remarkable results of EATNE (Q2)}

\subsubsection{Node Classification}
In this section, we conduct a multi-class classification on three graph datasets with different training ratios to simulate the semi-supervised scene. The results are organized and presented on Table. \ref{cora} - \ref{wiki}. It can be observed that EATNE achieves impressive performance, especially the best average classification accuracy on 18 training ratios for all three graph datasets. The relative average error rate of node classification is reduced by 7.2\% on Cora, 3.7\% on Citeseer and 3.1\% on Wiki. Moreover, we obtain the best results on 45/54 trials, demonstrating the benefit of Sine-based loss function in semi-supervised scenarios.
Note that although Cora and Citeseer are similar in network scale and number of classes, the imbalanced class distribution contributes to the difference of classification accuracy between Cora and Citeseer.
By the way, to simulate the scene that the graphs only contain structure information, we generate random features as inputs for GraphSage, or it will be unfair to other methods if we use original node features.

\subsubsection{Link Prediction}
Link prediction plays an essential role in real-world applications for the capacity of mining underlying relations between entities. We carry out link prediction on four graph datasets to demonstrate the effectiveness of our methods against other state-of-the-art methods. The results are summarized on Table. \ref{tab:auc}. As we can see that EATNE achieves competitive results on all four datasets compared with other methods. Specifically, based on AdvTNE, the AUC score increases by $2.63\%$ on Cora, $6.67\%$ on Citeseer, $8.22\%$ on Wiki and $8.46\%$ on PubMed, which shows that the Sine-based loss function can further promote performance of models applying APP. Besides, EATNE also performs well on the largest network datasets PubMed, revealing its great potential in scaling to large-scale graphs.

%% file: src/related_work.tex
\section{Related Work}
\textbf{Network Embedding} can be categorized into inductive learning methods \cite{hamilton2017inductive} and transductive learning methods \cite{Tang2015LINE,Cao2015GraRep,grover2016node2vec}. Inductive learning optimizes parameters of a well-designed neural network, following a node attributes-based message passing mechanism. Directly optimizing the node embeddings, transductive learning performs pretty well on the scene of unsupervised learning without node attributes and becomes the focus of this paper. Inspired by word embedding methods, some methods based on the skip-gram model \textit{word2vec} are proposed \cite{Perozzi2014DeepWalk, Tang2015LINE, grover2016node2vec}. Besides, \cite{levy2014neural} proved the equivalence between skip-gram models with negative-sampling and matrix factorization, which led to new proximity metrics under the matrix factorization proximity framework \cite{yang2015network, Cao2015GraRep}.

\begin{table}[!h]
\small
\centering
\caption{AUC score for link prediction}
\label{tab:auc}
\resizebox{0.5\textwidth}{!}{%
\begin{tabular}{c|c|c|c|c}
\toprule

Dataset & Cora & Citeseer & Wiki & PubMed \\ \midrule
DeepWalk & 0.628$\pm$0.014 & 0.528$\pm$0.010 & 0.525$\pm$0.007 & 0.503$\pm$0.001 \\
LINE & 0.606$\pm$0.015 & 0.514$\pm$0.009 & 0.506$\pm$0.004 & 0.502$\pm$0.002 \\
node2vec & 0.626$\pm$0.011 & 0.522$\pm$0.010 & 0.526$\pm$0.004 & 0.506$\pm$0.003 \\
GF & 0.500$\pm$0.003 & 0.505$\pm$0.011 & 0.507$\pm$0.003 & 0.502$\pm$0.004 \\
GraRep & 0.501$\pm$0.002 & 0.502$\pm$0.009 & 0.506$\pm$0.005 & 0.503$\pm$0.003 \\
GraphSage & 0.504$\pm$0.002 & 0.500$\pm$0.006 & 0.501$\pm$0.005 & 0.500$\pm$0.002 \\
AdvTNE & 0.647$\pm$0.008 & 0.520$\pm$0.010 & 0.532$\pm$0.004 & 0.523$\pm$0.006 \\
Cleora & 0.641$\pm$0.025&0.526$\pm 0.022$ &0.524$\pm$0.026& 0.512$\pm$0.004 \\

\midrule
\textbf{EATNE}&\textbf{0.665$\pm$0.004} &\textbf{0.544$\pm$0.008} &\textbf{0.566$\pm$0.005} &\textbf{0.551$\pm$0.002} \\ 
\bottomrule
\end{tabular}
}
\end{table}

\noindent \textbf{Adversarial Training} is introduced in Computer Vision, Speech Recognition and Natural Language Processing for promotion of model robustness \cite{Miyato2015distributional,madry2017towards,shafahi2019adversarial,sun2018domain,liu2019adversarial,drexler2018combining,miyato2018virtual,zhu2019freelb,liu2020adversarial}.
\cite{Miyato2015distributional} proposes virtual adversarial training on the basis of optimization of KL-divergence based robustness against local perturbations.
Adversarial training also achieves impressive performance on supervised and semi-supervised learning scenes \cite{miyato2018virtual} for the design of unlabelled adversarial direction. \cite{shaham2018understanding} explains its mechanism by the perspective of robust optimization. 
Researchers introduce Adversarial Training into network embedding for its great success. Some works apply perturbations to network structures \cite{dropedge,topology},
others on the node or edge attributes \cite{edgefeature,latentfeature}.
It is worth mentioning that Adversarial training based network embedding methods are different from Generative Adversarial Network (GAN) based ones \cite{wang2018graphgan,hong2019gane,progan}.
GAN based network embedding methods utilize a mini-max game between connectivity generators and corresponding discriminators to obtain representations.

Although some researches introduce Adversarial Training to generate adversarial perturbations on networks in various ways, few of them pay attention to perturbations in parameter space.
AdvTNE \cite{dai2019adversarial} integrates adversarial training with network embedding in a different manner that perturbations are added not in input space but parameter space. Inspired by this idea, in computer vision, some researches introduce adversarial perturbations on parameters (APP) for better generalization ability. \cite{foret2020sharpness} adopts APP to seek for parameters where the training loss of the entire neighborhoods are very low; \cite{wu2020adversarial} utilizes APP from the perspective of flattening the weight loss landscape; \cite{zheng2020regularizing} regards APP as a regularization to reach a flatter minima of a maximized empirical loss. To summarize, APP in AdvTNE has broadened the way of designing adversarial training so that it is well worth investigating the mechanism behind it, which is the main idea of this paper.

%% file: src/conclusion.tex
\section{Conclusion}
In this paper, we study APP theoretically from an optimization perspective, explore the reason for its impressive effectiveness, and propose a new activation based on AdvTNE. We first prove that APP can be interpreted as an optimization method and analyze its relationship with Momentum. One reason for its outstanding performance is that it provides momentum information to accelerate optimization on the saturation region. Detailed experiments verify the similarity between APP and Momentum on node classification tasks. To address the saturation region problem caused by Sigmoid activation, we design a new activation to obtain high-quality representations more easily.  
Our method achieves satisfactory results on four data sets on both node classification and link prediction tasks, which practically demonstrates the superiority of our activation. 
Furthermore, the Sine-based activation is expected to address limitations brought by Sigmoid-based activation in similar scenarios. Especially, it is reasonable to promote it to word embedding methods as they have similar paradigms and they both obey the power-law distribution. 
We will apply APP as a new optimization method and the Sine-based activation to tackle the Sigmoid saturation region problem in the aforementioned fields in future work.

%% file: src/appendix.tex
\clearpage
\setcounter{theorem}{0}
\section{Appendix}
\subsection{Proof of Theorem 1}
\begin{theorem}
\label{theorem:optimization}
Given an unconstrained differentiable loss function $L(\Theta)$, applying APP strategy with gradient descent is equivalent to iterative optimization based on the following strategy when updating the parameter $\Theta$:

\begin{equation}
\begin{split}
\label{equ:advtne_strategy}
    \Theta_{t+1} &= \Theta_t - \epsilon \Delta \Theta_t,\\
    \Delta \Theta_t =  \nabla & L(\Theta_t) + \lambda \nabla L(\Theta_t + n^{(adv)}_t)
\end{split}
\end{equation}

where $\Theta_t$ is the value of $\Theta$ in the $t$-th iteration, $\epsilon$ is the learning rate, and $n^{(adv)}_t$ is the adversarial perturbation which equals $\mathop{\arg\max}_{n, ||n|| \leq \rho} L(\Theta_t + n)$.
\end{theorem}
\begin{proof}
 Given the parameter at current step $t$ as $\Theta_t$, firstly we expect to maximize $L(\Theta_t+n)$ in the adversarial perturbation generation step:
 \begin{equation}
     n^{(adv)}_t=\mathop{\arg\max}_{n, ||n|| \leq \rho} L(\Theta_t + n).
\end{equation}
 As soon as $n^{(adv)}_t$ is determined, the objective is a function only related to $\Theta_t$ as $L(\Theta_t)+\lambda L(\Theta_t+n^{(adv)}_t)$ 
 
We optimize this objective through gradient descent with learning rate $\epsilon$:
 \begin{equation}
     \begin{split}
         \Theta_{t+1} &= \Theta^{t} -  \epsilon\Delta \Theta_{t},\\
 \Delta \Theta_{t} = \nabla &L(\Theta_{t}) +\lambda \nabla L(\Theta_{t} + n^{(adv)}_t).
     \end{split}
 \end{equation}
This updating procedure can be considered from another aspect: considering the objective function $L(\Theta_t)$, when updating $\Theta_t$, an additional term $\lambda \nabla L(\Theta_{t} + n^{(adv)}_t)$ is also utilized besides the gradient $\nabla L(\Theta_{t})$. As a result, it is reasonable to regard it as an optimization method.
Thus, the theorem is proved.
\end{proof}

\subsection{Proof of Theorem 2}
\begin{theorem}
\label{theorem:advt_approx}
If (1) $L$ is L-smooth, i.e., $||\nabla L (\Theta_t) - \nabla L(\Theta_{t-1})||\leqslant l ||\Theta_t - \Theta_{t-1}||$ 
and (2) $\exists t, s.t. ||\nabla L(\Theta_t)||\leqslant\delta$ hold where $l,\delta$ are both finite positive constants, then the following inequality holds:
\begin{equation}
\label{eq:max}
    \left|\max_{||n||\leq\rho}\{L(\Theta_t+n)\}-L(\Theta_{t-1}) \right| \leqslant (1+ \lambda \epsilon l)\rho \delta + \epsilon(1+\lambda) \delta^2.
\end{equation}
\end{theorem}

\begin{proof}
First, $L(\Theta_{t})$ can be written as:
\begin{equation}
\label{eq:theta_t_l_approx}
    \begin{aligned}
    L\left(\Theta_{t}\right)&=L\left(\Theta_{t-1}\right)+\nabla L\left(\Theta_{t-1}\right)^{\top}\left(\Theta_{t}-\Theta_{t-1}\right)+\mathcal{O}\left(\left\|\Theta_{t}-\Theta_{t-1}\right\|^{2}\right)\\
    &\approx L\left(\Theta_{t-1}\right)+\nabla L\left(\Theta_{t-1}\right)^{\top}\left(\Theta_{t}-\Theta_{t-1}\right) \\
    &=L\left(\Theta_{t-1}\right)- \epsilon \nabla L (\Theta_{ t - 1 } ) ^ { \top } \left[ 
        \nabla L\left(\Theta_{t-1}\right)+\lambda \nabla L \left(\Theta_{t-1}+n_{t-1}\right)
    \right] \\
    &= L\left(\Theta_{t-1}\right)- \epsilon \nabla L (\Theta_{ t - 1 } ) ^ { \top } 
        \nabla L\left(\Theta_{t-1}\right)\\
    & -\lambda \epsilon \nabla L (\Theta_{ t - 1 } ) ^ { \top } \nabla L \left(\Theta_{t-1}+n_{t-1}\right).
    \end{aligned}
\end{equation}
We denote the last term $\nabla L (\Theta_{ t - 1 } ) ^ { \top } \nabla L \left(\Theta_{t-1}+n_{t-1}\right)$ as $*$ and approximate it as:
\begin{equation}
\label{eq:star_app}
\begin{aligned}
* &=\nabla L(\Theta_{t-1 })^{\top}
\left[
\nabla L \left(\Theta_{t-1}\right)+\nabla L\left(\Theta_{t-1}+n\right)-\nabla L\left(\Theta_{t-1}\right)\right]\\
&\leqslant \left\|\nabla L\left(\Theta_{t-1}\right)\right\|^{2}
+ \left\|\nabla L\left(\Theta_{t-1}\right)\right\| \cdot \left\|\nabla L\left(\Theta_{t-1}+n\right)-\nabla L \left( \Theta_{t-1}\right)\right\|\\
&\leqslant \delta ^{2}+\delta l\|\Theta_{t-1}+n-\Theta_{t-1}\|
\\&\leqslant\delta ^{2}+\rho \delta l.
\end{aligned}    
\end{equation}

Secondly, we approximate $\max_{||n||\leq\rho}\{L(\Theta_t+n)\}$ via a first-order Taylor expansion as:
\begin{equation}\label{eq:theta_tn_app}
\begin{aligned}
    &\max_{||n||\leq\rho}\{L(\Theta_t+n)\} \\
    \approx &L\left(\Theta_{t} \right) + \max_{||n||\leq\rho} \nabla L\left(\Theta_{t} \right)^{\top}\cdot n \\
    \approx &L\left(\Theta_{t} \right) + \rho \left\| \nabla L\left(\Theta_{t} \right)\right\|.
\end{aligned}
\end{equation}
The maximum is obtained when $n = \rho \nabla L\left(\Theta_{t} \right)/\left\| \nabla L\left(\Theta_{t} \right)\right\|$.
Thus, by combining Eq. \eqref{eq:theta_t_l_approx} - \eqref{eq:theta_tn_app} we have:
\begin{equation}
    \begin{aligned}
    &\left|\max_{||n||\leq\rho}\{L(\Theta_t+n)\}-L(\Theta_{t-1}) \right| \\
    \approx& \left| 
        L\left(\Theta_{t} \right) + \rho \left\| \nabla L\left(\Theta_{t} \right)\right\| - L(\Theta_{t-1})
    \right| \\
    \leqslant& \left| 
        L\left(\Theta_{t} \right) + \rho \left\| \nabla L\left(\Theta_{t} \right)\right\| - \{L\left(\Theta_{t}\right) + \epsilon \nabla L (\Theta_{ t - 1 } ) ^ { \top } 
        \nabla L\left(\Theta_{t-1}\right) + \lambda \epsilon *  \}
    \right| \\
    =&\left|\rho \left\| \nabla L\left(\Theta_{t} \right)\right\| -\epsilon \left\|\nabla L (\Theta_{ t - 1 } ) \right\|^2  - \lambda\epsilon *
    \right|\\
    \leqslant& \rho \delta + \epsilon \delta^2 + \lambda \epsilon \delta^2 + \lambda \epsilon \rho \delta l\\
     =& (1+ \lambda \epsilon l)\rho \delta + \epsilon(1+\lambda) \delta^2.
    \end{aligned}
\end{equation}
The theorem is proved.
\end{proof}

Based on Theorem. \ref{theorem:advt_approx}, when the objective is stuck in the saturation region or is about to converge, we have the following two corollaries:
\begin{customthmc}{2.1}
\label{cor1}
The perturbation in Eq. (3) can be approximated as:
\begin{equation}
    n^{(adv)}_t = \Theta_{t-1} - \Theta_t.
\end{equation}
\end{customthmc}

\begin{customthmc}{2.2}
\label{cor2}
The AdvTNE optimization methods can be formulated as:
\begin{equation}\label{eq:advt_momentum}
\begin{split}
    \Theta_{t+1} &= \Theta_t- \epsilon \Delta\Theta_t,\\
    \Delta \Theta_t = \nabla &L(\Theta_t) +\lambda \nabla L(\Theta_{t-1}).
\end{split}
\end{equation}
\end{customthmc}

\begin{proof}
When the objective is optimized to the saturation region or is about to converge, the bound of gradient norm $\delta$ is quite small. Thus, $\left|\max_{||n||\leq\rho}\{L(\Theta_t+n)\}-L(\Theta_{t-1}) \right| \approx 0$.
We have $\max_{||n||\leq\rho}\{L(\Theta_t+n)\}\approx L(\Theta_{t-1})$.
$\Theta_t+n \approx \Theta_{t-1}$ is the sufficient condition for $\max_{||n||\leq\rho}\{L(\Theta_t+n)\}\approx L(\Theta_{t-1})$. 
Given that we only need to find a suitable solution of $n$, so that we have:
\begin{equation}
 n^{(adv)}_t = \Theta_{t-1} - \Theta_t.
\end{equation}

Replacing the $n^{(adv)}_t$ in Theorem. \ref{equ:advtne_strategy}, we can easily obtain:
\begin{equation}
\begin{split}
  \Theta_{t+1} &= \Theta_t-\epsilon\Delta\Theta_t,\\
    \Delta \Theta_t =  \nabla &L(\Theta_t) + \lambda \nabla L(\Theta_{t-1}).
\end{split}
\end{equation}
These two corollaries have been proven.
\end{proof}

\subsection{Proof of Theorem 3}

\begin{theorem}
\label{theorem:minimum}
For sufficiently large embedding dimensionality $r$, the minimum for similarity of point pair $(v_i, v_j)$ for SGNE objective has the following form:
\begin{equation}
\begin{split}
\label{equ:line_converge}
    S_{ij}^+ = \frac{w_{ij}}{w_{ij}+(\frac{d_id_j}{D})k},
\end{split}
\end{equation}
where $d_i = \sum_j w_{ij}$ and $D = \sum_i d_i$
\end{theorem}
\begin{proof}
Simplify the objective of SGNE as:
\begin{equation}
\begin{split}
   L &=  \sum_{v_i \in V}  \sum_{v_j \in N_S(v_i)}  w_{ij}\left(\log S_{ij}^+  + k \cdot \mathbb{E}_{v_n \sim P_{n}(v)}[\log S_{in}^-]\right)\\
     & = \sum_{v_i \in V}  \sum_{v_j \in N_S(v_i)} w_{ij} \log S_{ij}^+ 
     + k \sum_{v_i \in V}  d_i \cdot \mathbb{E}_{v_n \sim P_n(v)}[\log S_{in}^- ]\\
     & = \sum_{v_i \in V} \sum_{v_j \in V} w_{ij} \log S_{ij}^+ + \frac{k d_j d_i}{D} \log (1 - S_{ij}^+).
    \end{split}
\end{equation}
Since embedding dimensionality $r$ is sufficiently large, each $S_{ij}^+$ can assume a value independently of the others \cite{levy2014neural}. We consider the local objective for $S_{ij}^+$:
\begin{equation}
    L_{ij} = w_{ij} \log S_{ij}^+ + \frac{k d_j d_i}{D} \log (1 - S_{ij}^+).
\end{equation}
Let the gradient $\nabla L_{ij}(S_{ij}^+)$ be equal to 0:
\begin{equation}
    \nabla L_{ij}(S_{ij}^+) = \frac{w_{ij}}{S_{ij}^+} + \frac{k d_j d_i}{D (1 - S_{ij}^+)} = 0.
\end{equation}
We have 
\begin{equation}
    S_{ij}^+ = \frac{w_{ij}}{w_{ij}+(\frac{d_id_j}{D})k}.
\end{equation}
Thus, the theorem has been proved.
\end{proof}

\subsection{Proof of Theorem 4}

\begin{theorem}
Considering a scale-free network that the degree $d$ of each node follows the power-law as $d\sim P(d) \propto d^{-\alpha}, \, \alpha \in (2, 3)$, the number of edges is $|E|$ and the number of negative samples is $k$. Given a threshold $\gamma \in (0,1)$, the following equation holds:
\begin{equation}
\label{equ:s}
    P(S^+_{ij} \geq \gamma) \approx \frac{(\alpha-1)^2}{{-\alpha+2}}\frac{\log |E|}{2|E|}\cdot(R^{\alpha-2}-1),
\end{equation}
where $R=\frac{\gamma k}{2(1-\gamma)|E|}$.
\end{theorem}
\begin{proof}
For brevity, let $t_{ij}=\frac{w_{ij}}{d_id_j}$ and $D = 2|E|$, and the probability can be represented as follows:
\begin{equation}
    P(S^+_{ij} \geq \gamma) = P(\frac{w_{ij}}{d_id_j} \geq  \frac{\gamma}{1-\gamma}\frac{k}{D}) = P(t_{ij}\geq R).
\end{equation}
Following the definition of cumulative probability and the law of total probability, we have:
\begin{equation}
    \label{eq:them4_mid}
    \begin{split}
        P(t_{ij}\geq R) & = \sum_{t_{ij}=R}^{\infty} p(t_{ij})
         = \sum_{t_{ij}=R}^{\infty}\sum_{c_{ij} = 1}^{\infty}p(t_{ij}|c_{ij})p(c_{ij}),
    \end{split}
\end{equation}
where $c_{ij}=d_id_j$. Since we use LINE objective, $w_{ij} \in \{0, 1\}$. Remind that $t_{ij}=0$ if $w_{ij}=0$, therefore, we only need to calculate the probability when $w_{ij}=1$ i.e. $t_{ij}=\frac{1}{c_{ij}}\geq R$, which leads to $c_{ij} \leq \frac{1}{R}$. Eq. \eqref{eq:them4_mid} can be rewritten as:
\begin{equation}
    \begin{split}
        P(t_{ij}\geq R)=\sum_{c_{ij}=1}^{\frac{1}{R}}\sum_{t_{ij}=R}^{\infty}p(t_{ij}|c_{ij})p(c_{ij})
    \end{split}
\end{equation}

Consider the summation: $\sum_{t_{ij}=R}^{\infty}p(t_{ij}|c_{ij},D)$, following the configuration model \cite{PhysRevE.95.052303}, given $d_i,d_j$ and $D$, it can be reduced as $\frac{c_{ij}}{D}$.

\begin{equation}
    \begin{split}
        P(t_{ij}\geq R) &=\sum_{c_{ij}=1}^{\frac{1}{R}}\frac{c_{ij}}{D}p(c_{ij}). \\
    \end{split}
\end{equation}
Remember that the degree follows Power-law distribution of which Probability Density Function is $f_d(d) = (\alpha - 1)d^{-\alpha}$ \cite{clauset2009power}, we have
\begin{equation}
\begin{split}
P(t_{ij}\geq R) &\approx \int_{1}^{\frac{1}{R}} \frac{c_{ij}}{D}f_c(c_{ij})\, dc_{ij} \\
&= \frac{1}{D}\int_{1}^{\frac{1}{R}} c_{ij}\,dc_{ij}\int_{1}^{|E|}\frac{f_d(\frac{c_{ij}}{d_j})f_d(d_j)}{d_j}\, dd_{j}\\
&= \frac{1}{D}\int_{1}^{\frac{1}{R}} c_{ij}\,dc_{ij}\int_{1}^{|E|}\frac{(\alpha-1)^2c_{ij}^{-\alpha}}{d_j}\, dd_{j}\\
&= \frac{(\alpha-1)^2}{D}\int_{1}^{\frac{1}{R}}c_{ij}^{-\alpha+1}\log |E|\,dc_{ij}\\
&= \frac{(\alpha-1)^2}{{-\alpha+2}}\frac{\log |E|}{2|E|}\cdot(R^{\alpha-2}-1)
\end{split}
\end{equation}
Thus, the theorem has been proved.
\end{proof}